\documentclass{article}

\usepackage[nonatbib, main, preprint]{neurips_2026}
\usepackage[utf8]{inputenc} 
\usepackage[T1]{fontenc}    
\usepackage{hyperref}       
\usepackage{url}            
\usepackage{booktabs}       
\usepackage{amsfonts}       
\usepackage{nicefrac}       
\usepackage{microtype}      
\usepackage{xcolor}         
\usepackage{graphicx}
\usepackage{amsmath}
\newtheorem{thm}{Theorem}

\title{Diffusion Models, Denoiser Architecture and Creativity}

\author{%
  Itamar Levine, Yair Weiss\\
  The Hebrew University of Jerusalem\\
  \texttt{\{Itamar.Levine, Yair.Weiss\}@mail.huji.ac.il} \\
}

\begin{document}

\maketitle

\begin{abstract}
  The creativity of diffusion models refers to their ability to generate highly realistic images that are different from their training data. Creativity is somewhat surprising since  it is known that if the  denoiser used in the diffusion model  is the Bayes optimal denoiser for a given training set, then the model will simply copy the training samples. In this paper we present empirical and theoretical results that suggest that creativity in diffusion models is due to an interaction between the denoiser architecture and the target distribution. Theoretically, we give explicit forms for the distribution of generated samples as a function of the target distribution and the denoiser architecture for three different denoiser architectures (linear, polynomial, bottleneck).  Empirically, we show that small changes in the popular UNET denoiser architecture leads to very different forms of creativity, and these small changes often yield samples that are highly nonrealistic. Taken together, our results show that diffusion models will only be successful if the inductive bias of the denoiser architecture is in strong alignment with the true target distribution.
\end{abstract}

\section{Introduction}
The remarkable success of diffusion models\cite{conf/iclr/SongME21,10.5555/3495724.3496298,10.5555/3045118.3045358} lies in their "creativity"—their ability to generate highly realistic images that meaningfully deviate from their training data. This generative flexibility is somewhat surprising from a theoretical standpoint. It is known that if the denoiser used in a diffusion model acts as the Bayes optimal denoiser for a given dataset, the model simply copies the training samples, effectively collapsing the generated output into a uniform distribution over the training set.\cite{10.5555/3600270.3602196}

Recent literature has offered differing perspectives on how models escape this memorization to produce novel content. Some researchers, such as~\cite{kamb2025an}, posit that creativity arises primarily from the local properties of the model. Conversely, Lukoianov et al ~\cite{lukoianov2026locality} suggest that this generative diversity is fundamentally driven by the underlying data statistics. 

The research presented in this paper is based on our initial experiments that suggest that neither of these perspectives is sufficient to explain the success of diffusion models used in practice. Each row of figure~\ref{fig:unet} presents samples from a diffusion model trained on exactly the same dataset: a subset of 1000 images from CelebA. Each column shows images generated from exactly the same initial noise and each row used a slight variation on the commonly used unet architecture~\cite{10.1007/978-3-319-24574-4_28} as a denoiser.  

The unet architecture is a convolutional neural network that includes a contracting path (which consists of convolutions and pooling operations that reduce the spatial dimension of the image) and an expanding path that gradually increase the spatial dimension until it matches the original dimension. This architecture has a number of parameters including how many convolutional channels to use, how to handle the boundary conditions in the convolutions, and how many downsampling operations to perform before upsampling. We have found that each of these choices has a major effect on the generated samples.

The second row shows the samples when the unet parameters are set to their default values for generating facial images. Since we are only training on 1000 images, the results are far from the state-of-the-art but they exhibit the same creativity reported in previous papers~\cite{kadkhodaie2024generalization}), the images are rather realistic and are {\em not} identical to any training image. The top three row shows samples from the exact same architecture but when the number of pooling layers is varied (thus changing the size of the receptive field in the final output layer).  When the receptive field is too large (top), the samples are exact copies of the training images. When the receptive field is too small (third row), the samples appear highly distorted. 

The fourth row shows samples from the same architecture when we change the number of channels (now the samples lose their sharpness) and the bottom row shows samples when the convolution is performed with circular padding rather than zero padding. As discussed in~\cite{kamb2025an}, changing the padding makes the samples ignore the location information 

These simple experiments show that we cannot explain the creativity of diffusion models by appealing only to the dataset statistics or to the locality of the denoiser. First, all the samples are learned on exactly the same dataset and thus have access to the exact same statistics, and yet sometimes the samples are copies and sometimes they are not. Clearly, the architecture has a strong influence. Second, even when the denoiser is convolutional and local, the generated images may be copies of the training images (top row) or highly blurred (fourth row), indicating that the architecture's influence is not restricted to the size of the receptive field. 


\begin{figure}[h]
  \centering
  \includegraphics[width=0.9\textwidth]{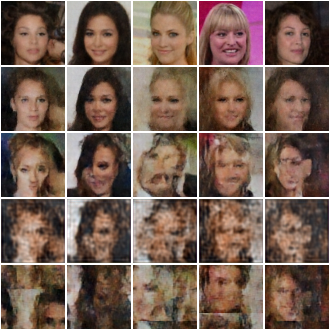}
  \caption{The impact of minor U-Net architectural variations on diffusion model outputs. All models were trained on the identical dataset (CelebA\cite{liu2015faceattributes}), and all images were generated using the exact same random seed to ensure a direct comparison. The results demonstrate that seemingly minor structural modifications lead to dramatically different generated distributions. From top to bottom, the rows display outputs from models utilizing: 4 pooling layers: Mostly memorizes and reproduces samples from the training set. 3 pooling layers: The standard, most commonly used configuration. 2 pooling layers. Reduced channel capacity (small channels). Circular padding. Crucially, despite possessing the exact same receptive field, the standard configuration (second row) and the bottom two configurations yield fundamentally different visual outcomes. This highlights that architectural choices inherently enforce different generated distributions, independent of the data statistics or receptive field size.}
  \label{fig:unet}
\end{figure}

Instead, we argue that creativity cannot be attributed to a single isolated factor. There is a deeply interconnected relationship between the empirical data statistics, the specific denoiser architecture, and the resulting generated distribution. Each architecture inherently enforces a distinct generated distribution, meaning diffusion models will only be successful if the inductive bias of the chosen denoiser strictly aligns with the true target distribution.

In this work, we explore this alignment through both theoretical and empirical lenses:

\begin{itemize}
    \item Theoretical Contributions (Section 2): We demonstrate theoretically how specific denoiser architectures derive distinct generated distributions. Furthermore, we prove that even minute modifications to the architecture can result in substantial, highly divergent shifts in the final distribution.
    \item Empirical Contributions (Section 3): We present several experiments supporting our theoretical claims, illustrating practically how varying the architecture—while holding the dataset constant—directly yields fundamentally different types of generated images.
\end{itemize}

\section{Analytic connection between architecture and sampled distribution}
\label{gen_inst}

Motivated by our initial experiments, we wish to answer the following question. Suppose we use training data sampled from a distribution $P_{target}$ to learn a denoiser with a particular architecture. We then use this denoiser within a diffusion model to generates samples from the distribution $P_{generated}$. What is the relationship between the architecture, $P_{generated}$ and $P_{sampled}$ ? 

All of the results are based on what we call the "score matching theorem":
\begin{thm}
Let $y=x+\eta$ where $\eta$ is Gaussian noise with variance $\sigma^2$. Let $S_\theta(y,\sigma)$ be a denoiser function. If: 
\[
\nabla \log P(y) = \frac{s_\theta(y,\sigma)-y}{\sigma^2}
\]
Then:
\begin{itemize}
\item $S_\theta(y,\sigma)$ is the minimizer of $E(\|S(y)-x\|^2)$
\item Sampling from a diffusion model with $S_\theta(y,\sigma)$ as the denoiser, yields $P_{generated}=P_{target}$
\end{itemize}
\end{thm}

{\bf Proof:} This is just a restatement of the results of~\cite{10.5555/3045118.3045358,conf/iclr/SongME21} using our formulation.

The score matching theorem implies that if we were to train a neural network denoiser to minimize $E(\|S(y)-x\|^2$, the network will converge to a denoiser that matches the score and sampling from the diffusion model will give us samples from $P_{target}$. But this implication is only correct if the architecture of the neural network matches the form of the score function. What happens when the architecture is restricted? As we now show, the
 the two distributions $P_{\theta},P_{generated}$ may be very different. 

One such case was already presented in~\cite{kamb2025an}: when the denoiser $S(y)$ is constrained to be local, it cannot match the score exactly, and this leads to the following theorem:

\begin{thm}
    For a convolutional architecture with receptive field size $S$, $P_{generated}$ only depends on $P_{target}(Mx)$ 
where $Mx$ is a patch of size $S$ extracted from $x$.
\end{thm}

 We now examine three distinct architectural classes—linear, polynomial, and bottleneck networks—and derive explicit forms demonstrating how each shapes the generated output. To empirically validate our theoretical conclusions, we trained models on 2-D data for the majority of the architectures examined.

\subsection{Linear Denoiser}
Linear denoising serves as an ideal analytical starting point because it allows for a closed-form expression of the generated distribution, $P_{generated}$. Although it is a relatively simple model, previous research, such as Wang and Vastola~\cite{wang2023hiddenlinearstructurescorebased}, has demonstrated that in the high-noise regime, the behavior of most diffusion models closely approximates that of a linear denoiser.
\begin{thm}
    When training a linear denoiser, $P_{generated} = \mathcal{N}(\mu, \Sigma)$ where $\mu=E[P{_{target}}]$ and $\Sigma=Cov(P_{target})$.
\end{thm}

{\bf Proof:} This follows from the classic result~\cite{Kay97} that learning a linear denoiser is mathematically equivalent to an optimal MMSE denoiser while assuming that the distribution is Gaussian. Thus the linear denoiser will match the score of a Gaussian distribution, and by the score matching theorem, $P_{generated}$ will be a Gaussian with the appropriate mean and covariance.

This theoretical framework leads directly to a crucial observation: if two distinct target distributions share the exact same mean and covariance, a linear denoiser will produce an identical $P_{generated}$
  for both. To empirically validate this observation, we constructed two distinct 2-dimensional datasets—a 3-component Gaussian Mixture Model (3GMM) and a 2-component Gaussian Mixture Model (2GMM)—that were explicitly designed to have matching means and covariances, and trained a linear denoiser on each.

\begin{figure}[h]
  \centering
  \includegraphics[width=\textwidth]{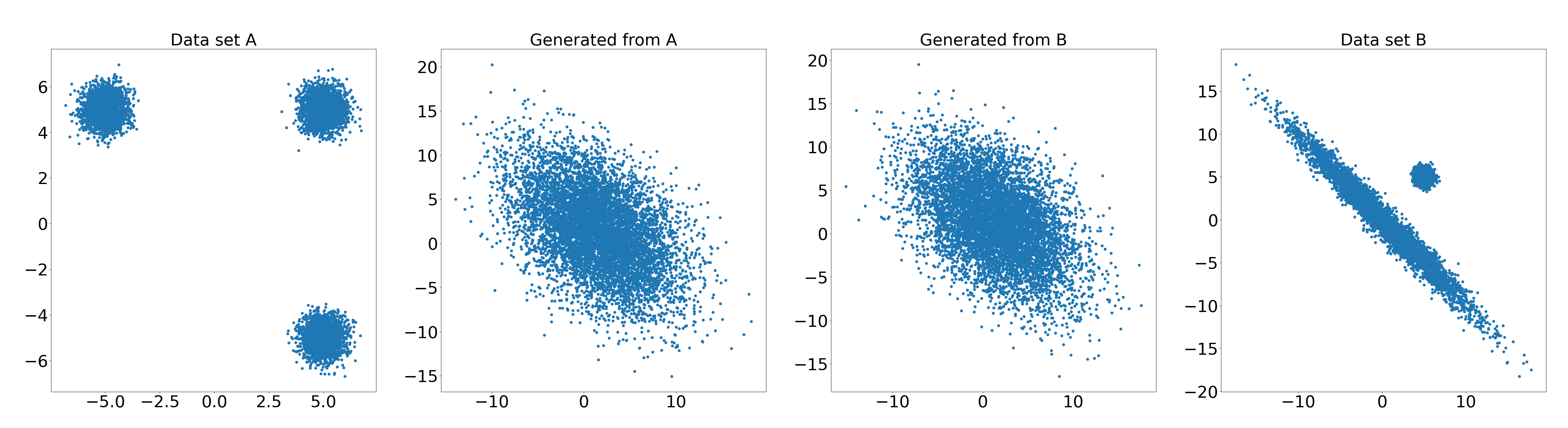}
  \caption{Two different data sets with the same mean and covariance (most left and right) and data generated using diffusion process learn by linear denoiser.}
\end{figure}

We now consider a linear denoiser with a bottleneck. The linear denoiser predicts the clean signal $x$ from the noisy signal $y$ by multiplying $y$ by a linear matrix $W$. But unlike the previous case is a product of two rectangular matrices, each of which has rank $K$: $W_{n \times n} = A _{n \times K} B_{K \times n}$ where $K$ is much smaller than $n$ (the dimension of the input). 

\begin{thm}
    when training a linear denoiser with bottle neck of size K, $P_{generated} = \mathcal{N}(\mu, \Sigma)$ where $\mu=E[P{_{target}}]$ and $\Sigma$ is the rank K approximation of $Cov(P_{target})$.
\end{thm}
\label{thm4}

{\bf Proof:}
The proof for this  theorem relies on the well-known connection between linear autoencoders and principal component analysis (e.g.~\cite{books/daglib/0033642}). In the \hyperref[proof4]{appendix}, we extend this connection for denoising, and show that the rank constrained linear denoiser will converge to the optimal linear denoiser assuming that the data was generated by a Gaussian whose covariance is the rank K approximation to the covariance of the data.

\subsection{Polynomial denoiser}

We define a polynomial architecture as a sequence consisting of a linear layer, followed by an entry-wise power-of-k activation function, and a final linear layer. By assuming the linear layers have infinite width, this architecture can represent any polynomial of degree k. Theoretically, as the degree k increases, the model can approximate any continuous function with arbitrary precision.

\begin{thm}
  when training a polynomial denoiser of degree $d$ $P_{generated}$ is a log-polynomial distribution with degree $d+1$, and $P_{generated}$ only depends on the first $2d$ moments of $P_{target}$.  
\end{thm}
\label{thm5}

{\bf Proof:}
The first part of this claim is based on the score matching theorem. When the score function is represented by a polynomial of degree $d$,  integrating the score function implies that $\log P(y)$ must be a polynomial of degree $d+1$ for any $\sigma$. In particular, as $\sigma \rightarrow 0$ this implies that $\log P(x)$ will also be a polynomial of the same degree.  The second part of the claim, regarding the dependence on the first $2d$ moments, is proven in the \hyperref[proof5]{appendix}.

To validate this, we trained polynomial denoisers of various degrees on the same dataset, yielding distinct generated distributions as shown in Figure 3. We note an important theoretical constraint: a log-probability distribution defined by a polynomial of odd degree is not integrable, as the integral over the entire domain would be infinite. Consequently, to ensure a valid probability density, we restricted our experiments to polynomial denoisers of odd degree $d$ resulting in log-distributions of even degree $d+1$.

In this figure, $P_{target}$ is a mixture of three Gaussians while $P_{generated}$ is very different. For a denoiser that is a polynomial of degree 1, the generated data is sampled from a Gaussian (this is of course equivalent to the linear denoiser).When the denoiser is a polynomial of degree $3$, the log density of the generated data is a polynomial of degree $4$ so it can represent bimodal data but still does not capture the true distribution,  

\begin{figure}[h]
  \centering
  \includegraphics[width=\textwidth]{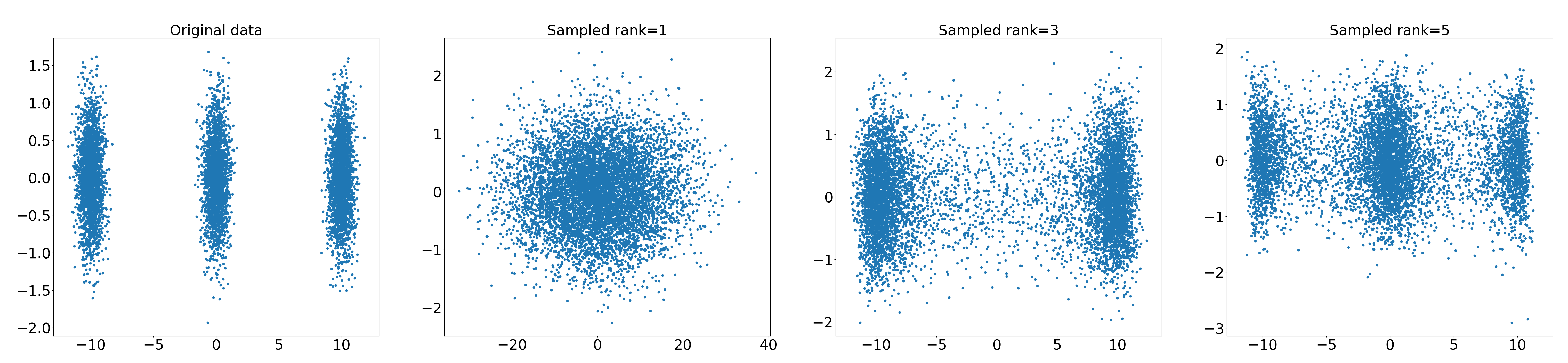}
  \caption{A dataset sampled from a mixture of three Gaussians, and the data sampled from $P_{generated}$ where the difussion process learn by a polynomial denoiser with different degrees.}
\end{figure}

\subsection{Bottleneck denoiser}
We define a bottleneck architecture as a neural network that consists of a deep and wide encoder that maps the input to an $h$ dimensional vector and then a deep and wide decoder.It is worth noting that the majority of neural networks currently employed for denoising naturally fall under our formal definition of a bottleneck architecture. Across these models, the primary differentiating structural factor is the dimensionality of the bottleneck itself, which we denote as $h$. The fundamental principle underlying this classification is that neural networks act as continuous and differentiable functions. When such a mapping passes through a restricted bottleneck of size $h$, the resulting output inherently satisfies the mathematical requirements of a $h$-rank manifold.

\begin{thm}
    When training a bottle neck architecture samples of $P_{generated}$ lie on an $h$ dimensional manifold. If $P_{target}$ does not lie on an $h$ dimensional manifold, then $P_{target} \neq P_{generated}$.
\end{thm}

{\bf Proof:} This follows from the simple fact that for any $\sigma$, the output of the denoiser must fall on an $h$ dimensional manifold. In particular this also holds for $\sigma \rightarrow 0$

To empirically investigate this property, we designed and trained a 2-D bottleneck architecture comprising four linear layers. The network first projects the 2-dimensional input into a relatively high-dimensional hidden space, compresses it through the bottleneck dimension h, expands it back to the high-dimensional space, and finally projects it back down to the original input dimension.

Specifically, we evaluated A model with a bottleneck dimension of h=1, utilizing a high-dimensional expansion of 64.

As presented in Figure 4, the experimental results clearly demonstrate that the intrinsic dimensionality of the generated output is strictly correlated with the bottleneck dimension $h$. By constraining the network's capacity at the bottleneck, we directly dictate the topological rank of the resulting generated distribution.

\begin{figure}[h]
  \centering
  \includegraphics[width=0.5\textwidth]{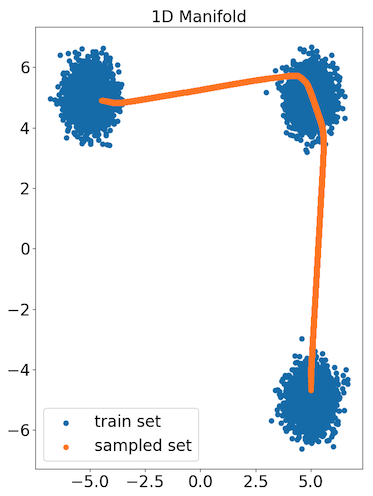}
  \caption{A 3 GMM data set, and the data sampled from $P_{generated}$ where the diffusion process learn by a bottle neck architecture with bottle neck of size 1.}
\end{figure}





\section{Experiments}

In this section, we train diffusion models on the CelebA dataset, employing the various architectures analyzed theoretically in Section 2. It is important to note that the primary objective of these experiments is not to achieve state-of-the-art image quality or to generate perfectly realistic novel images. Rather, our goal is to empirically demonstrate the  diverse ways in which specific architectural decisions dictate the behavior of the model and fundamentally shape the resulting generated distribution.

\subsection{Linear denoiser}
We first evaluated the behavior of the linear denoiser by introducing bottleneck constraints. Specifically, we trained linear models with bottleneck dimensions of $h=10$, $h=100$, and a full-capacity baseline with no bottleneck. As expected, restricting the capacity acts similarly to a dimensionality reduction process. By constraining the network, it is forced to discard the eigenvectors corresponding to the lowest eigenvalues; consequently, the generated images lose fine details and sharp edges. Conversely, the eigenvectors associated with the highest eigenvalues are preserved, ensuring that the broad structural shapes and primary color distributions of the images are maintained.

\begin{figure}[h]
  \centering
  \includegraphics[width=\textwidth]{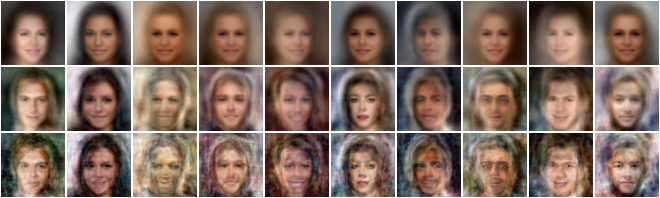}
  \caption{Samples from $P_{generated}$ learn by linear denoiser with bottle neck of size 10 (top row), 100 (second row) and no bottle neck (bottom row). each column have generated with the same starting noise. even after the bias of a linear denoiser, the bottle neck size have big influence over the sampled results.}
\end{figure}

In our second experiment, we empirically validated our theoretical claim from Section 2 regarding distributions with matching moments. We trained an initial linear denoiser on the CelebA dataset and sampled an entirely new dataset from its resulting $P_{generated}$ . Subsequently, we trained a second linear denoiser exclusively on this newly synthesized dataset. We then compared the outputs of both models when conditioned on the exact same initial noise. As demonstrated in our theoretical analysis, the linear denoiser perfectly matches the first two moments of the training data. Therefore, because the newly generated dataset inherently shares the same mean and covariance as the original CelebA dataset, the $P_{generated}$
  of both models are theoretically identical. As illustrated in Figure 6, this results in the two models yielding remarkably similar visual outputs, strongly supporting our theoretical conclusions.

\begin{figure}[h]
  \centering
  \includegraphics[width=\textwidth]{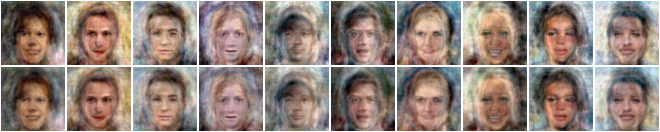}
  \caption{First row: samples from $P_{generated}$ learn by linear denoiser over CelebA. second row: samples from $P_{generated}$ learn by linear denoiser over samlped data set from $P_{generated}$ learn by linear denoiser over CelebA.}
\end{figure}

\subsection{Polynomial denoiser}
To simplify the training process for our polynomial architecture, we utilized a fixed random matrix for the initial linear projection, setting its width to 1024. The training process was then formulated as an optimization problem defined by the equation $A(Ry)^k=x$, where $x$ represents the clean target data, $y$ denotes the noisy input data, $R$ is the fixed random matrix, and $A$ is the learned weight matrix. Because we are minimizing the Mean Squared Error (MSE) with respect to the final matrix $A$, this specific optimization problem yields a closed-form solution. Consequently, this formulation allows the model to efficiently learn high-degree polynomial mappings without the typical instability of deep, non-linear backpropagation.

Using this approach, we trained polynomial denoisers of degrees $k=3$, $k=9$, and $k=27$. Samples generated from these respective distributions are presented in Figure 7. Notably, we observed that the 27-degree polynomial model exhibits a dual behavior: while some of its generated outputs closely resemble strictly memorized samples from the training set, other outputs are entirely novel images. 

\begin{figure}[h]
  \centering
  \includegraphics[width=\textwidth]{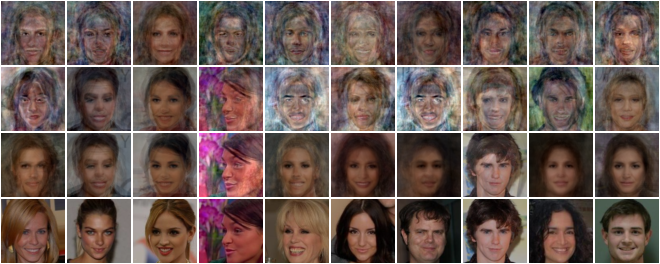}
  \caption{Samples from $P_{generated}$ learn by a polynomial denoiser over CelebA. we used different degree (from up: 3, 9, 27). the last row is the nearest neighbor of the 27-degree polynomial result. When using high degree, the polynomial denoiser create both creativity images and memorized images from train set. }
\end{figure}

Furthermore, we applied this methodology at the patch level to isolate and evaluate the specific influence of locality on the generative process. To achieve this, we followed~\cite{kamb2025an} and implemented a dynamic patch-size strategy, transitioning from larger patches in the high-noise regime to progressively smaller patches as the noise level decreased. Correspondingly, the dimensions of the random matrix were adjusted dynamically to align with these varying scales.

As illustrated in Figure 8, the results reveal that while the introduction of locality introduces certain visual nuances absent in the global configurations, it remains insufficient for synthesizing high-quality, novel images. 

\begin{figure}[h]
  \centering
  \includegraphics[width=\textwidth]{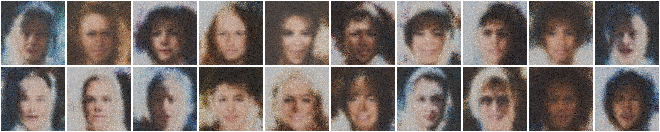}
  \caption{Samples from $P_{generated}$ learn by a patch polynomial denoiser over CelebA.}
\end{figure}

\subsection{Bottleneck}
As previously discussed, the standard Multi-Layer Perceptron (MLP) architecture aligns with our formal definition of a bottleneck network. To empirically test this, we trained three separate MLP networks with varying hidden layer capacities: $h=1000$, $h=500$, and $h=100$. Given that our experimental dataset was deliberately small (consisting of exactly 1000 samples), the network with $h=1000$ possessed a dimensional capacity perfectly matching the dataset size. Consequently, and rather unsurprisingly, this configuration led the model to strictly memorize and reproduce the training samples.

Conversely, as we restricted the capacity in the networks with $h=500$ and $h=100$, the models avoided strict memorization and instead exhibited distinct forms of generative creativity. As illustrated in Figure 9, employing progressively smaller hidden dimensions forces the generated data to lie on a lower-rank manifold, which inherently causes a proportional loss of specific, granular knowledge from the original dataset. Ultimately, this experiment demonstrates that simply utilizing a global network architecture is insufficient to guarantee memorization. Rather, the dimensional constraints of the bottleneck inherently dictate the tradeoff between memorizing the training data and generating novel, creative samples.

\begin{figure}[h]
  \centering
  \includegraphics[width=\textwidth]{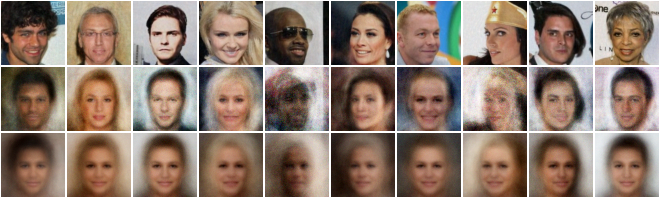}
  \caption{Samples from $P_{generated}$ learn by a bottle neck architecture with h=1000 (up), h=500 (middle) and h=100 (down). as expected, the top row is exact sample from train set. }
\end{figure}

\section{Related work}
Recent literature has extensively explored the internal mechanisms that enable diffusion models to produce high-quality images. Notably, Kadkhodaie et al. ~\cite{kadkhodaie2025unconditionalcnndenoiserscontain, kadkhodaie2024generalization} investigate the specific properties of U-Net and CNN denoisers, demonstrating that these architectures achieve strong generalization and synthesize realistic images by leveraging geometry-adaptive harmonic representations and sparse semantic features. While their work provides valuable intuition into how specific U-Net properties facilitate "good" image generation, our research shifts the focus toward a broader, mathematical formalization of the generative process itself. Rather than evaluating image realism through the lens of a single architecture, we mathematically characterize the direct connection between several distinct architectural classes (including linear, polynomial, and bottleneck models) and their resulting generated distributions. By doing so, we provide a theoretical framework that proves how overarching structural choices inherently enforce specific probability distributions, regardless of whether the resulting images align with human perception of realism.

A prominent line of research posits that the "creativity" and generalization capabilities of diffusion models stem primarily from their local network properties. Specifically, previous works—such as those by Ganguli et al.~\cite{kamb2025an} and Niedoba et al.~\cite{niedoba2025towards}—suggest that restricted receptive fields and local patch-based processing prevent models from memorizing global training structures, thereby forcing them to generate novel, cohesive compositions. While these theories highlight the importance of local features in synthesizing realistic images, our work demonstrates that locality alone is insufficient to explain or guarantee generative diversity. As shown in our empirical analysis (e.g., Figure 1), denoiser architectures possessing the exact same receptive field can yield drastically different generative outcomes, ranging from strict memorization to the production of completely novel distributions. 

Other recent investigations by Zeno et al. ~\cite{zeno2025when} have turned to simplified, shallow networks to probe the boundary between memorization and generalization in diffusion models. These studies share our basic approach of analyzing tractable, elementary models, but we focus on more general architectures.

\section{Conclusion and Limitations}
In this work, we highlight the impact of architectural inductive biases on the generated distribution. We demonstrate that, under equivalent settings and on the same dataset, both global and local architectures are capable of exhibiting true generative creativity as well as training set memorization. Furthermore, we provide analytical results that formalize the relationship between a model’s architectural design, its hyperparameter selection, and the behavior of the resulting distribution.

Our study entails two primary limitations. First, our mathematical analysis is constrained to relatively simple theoretical models. State-of-the-art generative networks, such as U-Net architectures, are inherently more complex; empirical evidence indicates that nuanced design details dictate the quality of their generated outputs. Second, our empirical evaluations were conducted on a relatively small dataset. While this deliberate choice allowed us to isolate and observe specific phenomena—such as the exact data replication detailed in Section 3.3—it inherently restricted our ability to showcase high-fidelity novel image generation, an objective that typically requires significantly larger training sets.

\newpage

\small

\small

\medskip


\appendix

\section{Theoretical proofs}

\subsection{Proof theorem 4}
\label{proof4}
Let $x \in \mathbb{R}^d$ be a signal vector with covariance $\Sigma_x = \mathbb{E}[xx^T]$. Let $y = x + \epsilon$ be the observed noisy signal, where $\epsilon \sim \mathcal{N}(0, \sigma^2 I)$ is Gaussian noise independent of $x$.

We consider a linear neural network with a bottleneck of dimension $k < d$. The network is represented by the weight matrices $W_1 \in \mathbb{R}^{k \times d}$ and $W_2 \in \mathbb{R}^{d \times k}$. The goal is to minimize the Mean Squared Error (MSE):

\begin{equation}
J(W_1, W_2) = \mathbb{E}[\|x - W_2 W_1 y\|^2]
\end{equation}

Let $L = W_2 W_1$ be the rank-$k$ linear operator. The objective is:

\begin{equation}
J(L) = \mathbb{E}[\|x - Ly\|^2] = \mathrm{tr}\left( \Sigma_x - 2L\Sigma_x + L(\Sigma_x + \sigma^2 I)L^T \right)
\end{equation}

Using the independence of $x$ and $\epsilon$:

\begin{itemize}
\item $\mathbb{E}[yx^T] = \mathbb{E}[(x + \epsilon)x^T] = \Sigma_x$
\item $\mathbb{E}[yy^T] = \mathbb{E}[(x + \epsilon)(x + \epsilon)^T] = \Sigma_x + \sigma^2 I = \tilde{\Sigma}$
\end{itemize}

To solve for the rank-$k$ constraint, we rewrite $J(L)$ by completing the square. Note that:

\begin{equation}
\|M - L \tilde{\Sigma}^{1/2}\|_F^2 = \mathrm{tr}(MM^T) - 2 \mathrm{tr}(L \tilde{\Sigma}^{1/2} M^T) + \mathrm{tr}(L \tilde{\Sigma} L^T)
\end{equation}

By setting $M = \Sigma_x \tilde{\Sigma}^{-1/2}$, the second term becomes $2 \mathrm{tr}(L \Sigma_x)$, matching our objective function.

Thus, the minimization is equivalent to:

\begin{equation}
\min_{\mathrm{rank}(L) \le k} \| \Sigma_x \tilde{\Sigma}^{-1/2} - L \tilde{\Sigma}^{1/2} \|_F^2
\end{equation}

Let the target matrix be $M = \Sigma_x \tilde{\Sigma}^{-1/2}$. Since $\Sigma_x$ and $\tilde{\Sigma}$ share the same eigenvectors $U = [u_1, \dots, u_d]$, the SVD of $M$ is:

\begin{equation}
M = \sum_{i=1}^d \frac{\lambda_i}{\sqrt{\lambda_i + \sigma^2}} u_i u_i^T
\end{equation}

According to the Eckart-Young-Mirsky theorem, the best rank-$k$ approximation of $M$ is the truncated sum:

\begin{equation}
(L \tilde{\Sigma}^{1/2})^* = \sum_{i=1}^k \frac{\lambda_i}{\sqrt{\lambda_i + \sigma^2}} u_i u_i^T
\end{equation}

To isolate $L$, we multiply by $\tilde{\Sigma}^{-1/2}$ on the right. Note that:

\[
\tilde{\Sigma}^{-1/2} = \sum_{j=1}^d \frac{1}{\sqrt{\lambda_j + \sigma^2}} u_j u_j^T
\]

Since the eigenvectors are orthonormal ($u_i^T u_j = \delta_{ij}$):

\begin{equation}
L^* = \left( \sum_{i=1}^k \frac{\lambda_i}{\sqrt{\lambda_i + \sigma^2}} u_i u_i^T \right)
\left( \sum_{j=1}^d \frac{1}{\sqrt{\lambda_j + \sigma^2}} u_j u_j^T \right)
= \sum_{i=1}^k \frac{\lambda_i}{\lambda_i + \sigma^2} u_i u_i^T
\end{equation}
Which is the k-approximation of the data's covariance.

\subsection{proof  theorem 5}
\label{proof5}
We now prove the second part of our claim. To help our proof, we'll look at a different form of the polynomial denoiser with specific rank k: $f(y) = Ah(y) + b$ where $h(y)$ is a monomial of y up to rank k, A is a matrix and b is a vector that minimize the MSE loss.
the optimal solution for a problem in this from is: $A=\Sigma_{x,h}\Sigma_h^{-1}$ and $b=\mu_x-A\mu_h$ with $\mu_x = E[X]$, $\mu_h = E[h(Y)]$, $\Sigma_{x,h}=Cov(X,h(Y))$ and $\Sigma_h=Cov(h(Y),h(Y))$. from definition is holds that:
\begin{equation}
    (\Sigma_h)_{i,j} = \Sigma_{(i_1,..,i_k)}\Sigma_{(j_1,..,j_k)}E[\underbrace{\prod_{l=1}^k(x_l+\eta_{l})^{i_l}}_{\text{the i'th monom of h(y)}}\prod_{l=1}^k(x_l+\eta_{l})^{j_l}]
\end{equation}
where $x_l$ is the index $l$ of the high-dimensional variable $X$, $\Sigma_{l=1}^ki_l = i$ and $\Sigma_{l=1}^kj_l = j$. with similar analysis we know that:
\begin{equation}
    (\Sigma_{x,h})_{i,j} = \Sigma_{(j_1,..,j_k)}E[(x_i+\eta_i)\prod_{l=1}^k(x_l+\eta_{l})^{j_l}]
\end{equation}
Due to the fact that the noise is independent with $X$ and with a little arithmetic computation we get that each entry of $\Sigma_h$ and $\Sigma_{x,h}$ is linear combination of elements in the form $E[\prod x^j_i]$ which is some moment of $X$. note that the rank of the highest moment in those matrices is $2k$.

\section{Experiments details}
All the experiments in this work made on the same data set CelebA with size=1000. The linear models (with or without bottle neck) sampled using 10 steps of the DDIM algorithm. the other models sampled using 50 steps. we used cosine scheduler for all of our models.
All of the neural networks we used train over 500 epochs with lr=1e-4 and a simple Adam optimizer. 
We used A100 Typed GPU with one worker, 22GB, and all experiments took between secods to several hours (~12 top).

All the used architectures are available here: \url{https://github.com/ItamarLevine/ArchitectureAndCreativity}

\end{document}